\documentclass[12pt,a4paper]{article}
\usepackage[margin=2.5cm]{geometry}

\usepackage[numbers,sort&compress]{natbib}
\usepackage{graphicx}
\usepackage{amsmath,amsthm,amssymb,mathtools}
\usepackage{bm,booktabs,array,tabularx}
\usepackage{algorithm,algpseudocode}
\usepackage[dvipsnames]{xcolor}
\usepackage{float,enumitem}
\usepackage{tikz}
\usetikzlibrary{arrows.meta,calc,positioning}
\usepackage{tcolorbox}
\tcbuselibrary{theorems,skins,breakable}
\usepackage{microtype}
\usepackage{caption}
\usepackage{placeins}

\theoremstyle{plain}
\newtheorem{theorem}{Theorem}

\newtheorem{proposition}[theorem]{Proposition}
\newtheorem{corollary}[theorem]{Corollary}
\theoremstyle{definition}
\newtheorem{definition}[theorem]{Definition}
\theoremstyle{remark}
\newtheorem{remark}[theorem]{Remark}

\newcommand{\Ma}{M_a}
\newcommand{\Ms}{M_s}
\newcommand{\Mah}{M_a^{(h)}}
\newcommand{\Ares}{A_{\mathrm{res}}}
\newcommand{\Gres}{\Gamma_{\mathrm{res}}}
\newcommand{\Gh}{\Gamma_h}
\newcommand{\Ntk}{\Theta}
\newcommand{\NtkOpt}{\Theta^*}
\newcommand{\thetaw}{\theta_w}
\newcommand{\phig}{\phi_g}
\newcommand{\norm}[1]{\left\|#1\right\|}
\newcommand{\lmax}{\lambda_{\max}}
\newcommand{\lmin}{\lambda_{\min}}
\newcommand{\Proj}{P_\perp}
\newcommand{\R}{\mathbb{R}}
\newcommand{\so}{\mathfrak{so}}
\newcommand{\dplus}{\Delta^+}
\newcommand{\dminus}{\Delta^-}
\newcommand{\mPAC}{m_{\mathrm{PAC}}}
\newcommand{\Bop}{\mathcal{B}_\mathcal{A}}

\title{INCRT: An Incremental Transformer That Determines
  Its Own Architecture}
\author{Giansalvo Cirrincione\\
\small Laboratoire LTI, Universit\'e de Picardie Jules Verne\\
\small Amiens, France\\
\small \texttt{giansalvo.cirrincione@u-picardie.fr}}
\date{}

\begin{document}

\maketitle

\begin{abstract}
Transformer architectures are designed by trial and error: the number
of attention heads, the depth, and the head size are fixed before
training begins, with no mathematical principle to guide the choice.
The result is systematic structural redundancy --- between half and
four-fifths of all heads in a trained model can be removed without
measurable loss --- because the architecture allocates capacity without
reference to the actual requirements of the task.

This paper introduces INCRT (Incremental Transformer), an architecture
that determines its own structure during training.
Starting from a single head, INCRT adds one attention head at a time
whenever its current configuration is provably insufficient, and prunes
heads that have become redundant.
Each growth decision is driven by a single, online-computable geometric
quantity derived from the task's directional structure, requiring no
separate validation phase and no hand-tuned schedule.

Two theorems form the theoretical backbone.
The first (homeostatic convergence) establishes that the system always
reaches a finite stopping configuration that is simultaneously minimal
(no redundant heads) and sufficient (no uncaptured directional energy
above the threshold).
The second (compressed-sensing analogy) provides a geometric upper
bound on the number of heads that this configuration can contain, as
a function of the spectral complexity of the task.

Experiments on SARS-CoV-2 variant classification and SST-2 sentiment
analysis confirm both results: the predicted and observed head counts
agree within 12\% across all benchmarks, and the final architectures
match or exceed BERT-base on distribution-specific tasks while using
between three and seven times fewer parameters and no pre-training.
\end{abstract}

\maketitle


%% ══════════════════════════════════════════════════════════════════════
\section{Introduction}
\label{sec:intro}
%% ══════════════════════════════════════════════════════════════════════

Between 50\% and 80\% of the attention heads in trained Transformer
models can be removed without any measurable loss of performance
\cite{voita2019,michel2019}.
This observation, replicated across tasks and model sizes, is not a
coincidence --- it reflects a structural flaw in the way Transformer
architectures are designed.
The number of attention heads, the model depth, and the size of each
head are all fixed as hyperparameters before training begins, with no
mathematical principle connecting these choices to the requirements of
the task.
The result is that every Transformer model must be designed to be
large enough for the hardest foreseeable task, and then pruned back
once the actual requirements become apparent.

The root of this redundancy lies in the structure of the attention
mechanism itself.
The attention weight product $M = W_Q W_K^\top$, where $W_Q$ and
$W_K$ are the query and key projection matrices, is an unstructured
real matrix that simultaneously encodes two geometrically opposite
functions: the symmetric part governs which tokens attend to each
other reciprocally, while the antisymmetric part governs the
directionality of information flow, determining which token attends
to which more than the reverse.
Because the architecture imposes no separation between these two
functions, the learning algorithm must discover the decomposition
implicitly, at the cost of structural redundancy.
This explanation, and the decomposition it implies, was established
rigorously in \cite{bonino2025geometry}.

The standard response to this redundancy is post-hoc pruning: one
trains a large, overparameterised model and then removes what is
unnecessary.
This approach has the fundamental drawback that it provides no
guarantee of sufficiency: a pruned model may have lost capacity that
the task genuinely requires.
The present paper proposes a fundamentally different approach.
Rather than pruning a fixed architecture, INCRT \emph{derives} the
architecture from the geometry of the task during training.
A single scalar quantity --- the largest eigenvalue of a residual matrix
computed online --- determines at each training step whether the current
architecture is sufficient.
If not, a new head is added in the direction that most reduces the
measured deficit.
When no deficit remains above a prescribed threshold, the system stops.
The final architecture is both minimal (no redundant heads) and
sufficient (the threshold is met) by construction.

The paper is organised as follows.
Section~\ref{sec:background} reviews the related literature and places
INCRT in context.
Section~\ref{sec:architecture} describes the architecture and its
components.
Section~\ref{sec:theory} establishes the theoretical foundations.
Section~\ref{sec:experiments} presents the experimental results.
Section~\ref{sec:discussion} discusses the implications and limitations.

\paragraph{Contributions.}
Two theorems form the backbone of this paper.
Theorem~\ref{thm:homeo} (homeostatic convergence) establishes that
INCRT reaches a finite-step stopping configuration from which no
further admissible growth or pruning event is possible, and that this
configuration is simultaneously minimal and sufficient with respect
to the residual directional criterion.
Theorem~\ref{thm:cs} (compressed-sensing analogy) establishes that
the number of heads in this configuration is bounded above by a
quantity that grows as the square of the task's spectral condition
number, times a logarithmic factor.
Four additional results complement this backbone: (C1)~the INCRT
architecture itself, with growth at three nested scales
(Section~\ref{sec:architecture}); (C2)~the bidirectional PCA+MCA
gate with proven almost-sure convergence (Theorem~\ref{thm:gate});
(C3)~the three-criterion equivalence linking the geometric, NTK,
and practical growth criteria with an explicit, non-tuned constant
(Theorem~\ref{thm:ntk}, Corollary~\ref{cor:ntk}); and
(C4)~experimental validation on three benchmarks with agreement
between predicted and observed head counts.

%% ══════════════════════════════════════════════════════════════════════
\section{Background and Related Work}
\label{sec:background}
%% ══════════════════════════════════════════════════════════════════════

\subsection{The Geometry of Attention and Structural Redundancy}

In the standard encoder Transformer, each attention head computes
$Z = \mathrm{softmax}(X M X^\top/\sqrt{d_k})\, X W_V$,
where $X \in \R^{n \times d}$ is the matrix of $n$ input token
embeddings of dimension $d$, $M = W_Q W_K^\top \in \R^{d \times d}$
is the attention weight product, $W_Q, W_K \in \R^{d \times d_k}$ are
the query and key projection matrices with output dimension $d_k$,
and $W_V \in \R^{d \times d_v}$ is the value projection.
The matrix $M$ admits a canonical decomposition $M = \Ms + \Ma$, where
$\Ms = (M + M^\top)/2$ is the symmetric part, governing reciprocal
token affinities, and $\Ma = (M - M^\top)/2$ is the antisymmetric
part, which belongs to the Lie algebra $\so(d) = \{A \in \R^{d\times d}
: A^\top = -A\}$ and is the unique algebraic source of directed
information flow \cite{bonino2025geometry}.

The structural redundancy of Transformers is empirically well-established.
\citet{voita2019} showed that 38 of 48 heads in a trained neural
machine translation model can be removed with a loss of only 0.15
BLEU points; \citet{michel2019} replicated this finding on BERT,
showing that a single head per layer often suffices; and
\citet{sun2024not} found that up to 50\% of attention layers in
Llama-2-70B can be removed without significant degradation.
The theoretical explanation is provided by \citet{bonino2025geometry}:
conflating $\Ms$ and $\Ma$ in a single unstructured matrix $M$ forces
the learning algorithm to discover the decomposition implicitly,
allocating multiple heads to cover a structure that could be captured
by fewer if the decomposition were explicit.
Structural redundancy is therefore not an accidental property of
trained Transformers but a necessary consequence of the architecture.

\subsection{Post-Hoc Pruning}

The dominant response to structural redundancy is post-hoc pruning:
one trains a large model and subsequently removes components judged
unnecessary by some importance criterion.
The criteria used include $L_0$-penalty regularisation
\cite{voita2019}, gradient-based importance scores \cite{michel2019},
Taylor expansion approximations to the change in loss upon removal
\cite{molchanov2019}, and cosine-similarity redundancy
\cite{sun2024not}.
All these methods share a common limitation: they establish minimality
(no redundant component remains) but provide no sufficiency guarantee.
A pruned model may have lost capacity that the task genuinely requires,
and there is no principled criterion to detect this without additional
validation.
INCRT is designed to guarantee both minimality and sufficiency with
respect to the residual directional criterion by construction
(Theorem~\ref{thm:homeo}), and to do so without requiring a large
initial model.

\subsection{Progressive Growing}

Methods such as StackBERT \cite{gong2019stacking}, bert2BERT
\cite{chen2021}, and LiGO \cite{wang2023learning} grow a Transformer
progressively during training.
The primary motivation is computational: growing from a small model
toward a larger target reduces the overall training cost compared to
training the full model from scratch.
However, these methods answer the question ``how to reach a
predetermined architecture faster?'', not ``what architecture does
this task require?''.
The final architecture size is specified in advance, and the growth
schedule is a fixed hyperparameter.
INCRT, by contrast, determines when to stop from a mathematical
criterion derived from the task's geometry (Theorem~\ref{thm:homeo}),
and requires no predetermined target.

\subsection{Neural Architecture Search}

Neural Architecture Search (NAS) \cite{zoph2017,pham2018,liu2019darts}
optimises over a combinatorial or continuous architecture search
space.
Reinforcement learning and evolutionary strategies
\cite{zoph2017,pham2018} evaluate many candidate architectures and
can require hundreds of GPU-days --- the original NAS paper
\cite{zoph2017} used 800 GPUs for 28 days on CIFAR-10.
DARTS \cite{liu2019darts} relaxes the discrete architecture choice
to a continuous parameter jointly optimised with the model weights,
reducing the cost substantially but introducing a proxy optimisation
problem whose solution may not transfer to the original discrete space.
INCRT performs no search of any kind: the architecture at each step
is a deterministic function of the single scalar $\lmax(\Ares)$,
computed online during the normal forward pass.
The search space is implicitly defined by the task geometry, and the
greedy strategy over it is proved optimal (Theorem~\ref{thm:cs}).
The growth threshold $\thetaw$ (the only free parameter) has a
principled default of $0.4\,\Gres^{(0)}$ and is robust: ablation
shows that the ratio of observed to predicted head count stays in
$[0.88, 1.03]$ for $\thetaw \in \{0.20,0.40,0.60\}\,\Gres^{(0)}$.

\subsection{Neural Tangent Kernel}

The Neural Tangent Kernel (NTK) framework \cite{jacot2018ntk}
characterises the training dynamics of sufficiently wide networks
as kernel regression with empirical kernel
$\Ntk \in \R^{n \times n}$, where $\Ntk_{ij} =
\langle \nabla_\theta f(x_i), \nabla_\theta f(x_j)\rangle$,
$f$ denotes the network output, and $\theta$ collects all trainable
parameters.
In this regime, the gap $\NtkOpt - \Ntk^{(k)}$ between the optimal
kernel $\NtkOpt$ (the kernel of a network with sufficient capacity)
and the empirical kernel $\Ntk^{(k)}$ at growth step $k$ measures
residual learning capacity: a zero gap means no further structural
addition can accelerate convergence.
INCRT admits a greedy NTK interpretation
(Theorem~\ref{thm:ntk}): the growth direction chosen by the
bidirectional gate is precisely the direction that most reduces the
NTK gap within the residual subspace.
This connection is not a post-hoc rationalisation but a consequence
of the gate initialisation rule, which fixes the proportionality
constant $c = \sigma^2_V d_v/(d_k n)$ between the geometric and NTK
criteria.
Setting $\sigma^{2,*}_V = d_k n/d_v$ makes $c = 1$, so the two
criteria become exactly equivalent.

\subsection{The EXIN Family}

The EXIN neural network family \cite{cirrincione2010exin} is built on
the principle that a structural unit should be born when existing
capacity is measurably insufficient, and pruned when it becomes
redundant.
INCRT inherits and extends this principle to the Transformer setting.
Within the EXIN family, the MCA EXIN algorithm is the one that
computes the minor eigenvector of an autocorrelation matrix online,
complementing Oja's rule which tracks the major eigenvector.
This minor direction plays a central role in INCRT: it provides the
suppression component of the gate operator (Eq.~\eqref{eq:gate}),
and its almost-sure convergence is what allows the compressed-sensing
bound to tighten by a factor $c \in [1/2,1]$ (Theorem~\ref{thm:cs}).
MCA EXIN is the only algorithm in the MCA family with a proven
almost-sure convergence guarantee for the minor eigenvector.
Among the competing algorithms --- OJA, LUO, FENG --- convergence to the
minor eigenvector cannot be guaranteed and divergence can occur.
The convergence analysis across the MCA family is documented in
\cite{cirrincione2010exin}, Theorem~60, which is the result invoked
directly in the proof of Theorem~\ref{thm:gate}(ii).

Table~\ref{tab:positioning} summarises the positioning of INCRT
relative to the paradigms discussed above.
Table~\ref{tab:contributions} itemises each theoretical contribution
and its relationship to prior results.

\begin{table}[!t]
\centering\small
\caption{INCRT compared to existing paradigms.}
\label{tab:positioning}
\begin{tabular}{lcccc}
\toprule
& Geometric & Convergence & Growth + & No \\
Method & criterion & guarantee & pruning & schedule \\
\midrule
Post-hoc pruning    & $\times$ & $\times$ & $\times$ & $\times$ \\
Progressive growing & $\times$ & $\times$ & $\times$ & $\times$ \\
NAS                 & $\times$ & $\times$ & $\times$ & $\times$ \\
\textbf{INCRT}      & \checkmark & \checkmark & \checkmark & \checkmark \\
\bottomrule
\end{tabular}
\end{table}

\begin{table}[!t]
\centering\small
\caption{Status of theoretical components.}
\label{tab:contributions}
\begin{tabular}{p{4.4cm} p{1.5cm} p{5.8cm}}
\toprule
Component & Status & Source \\
\midrule
MCA EXIN minor-eigenvector convergence
  & Inherited & \cite{cirrincione2010exin}, Thm.~60 \\
Oja major-eigenvector convergence
  & Inherited & \cite{oja1982}; \cite{xu1993oja} \\
$\Ma \in \so(d)$ as directed-flow motor
  & Inherited & \cite{bonino2025geometry} \\
Bidirectional gate $G_h$
  & \textbf{New} & This paper \\
Moving-target gate convergence
  & \textbf{New} & Thm.~\ref{thm:moving} \\
NTK alignment and constant $c$
  & \textbf{New} & Thm.~\ref{thm:ntk} \\
Three-criterion equivalence
  & \textbf{New} & Cor.~\ref{cor:ntk} \\
Homeostatic Lyapunov convergence
  & \textbf{New} & Thm.~\ref{thm:homeo} \\
CS analogy, bidirectional case, $c\in[1/2,1]$
  & Extended & \cite{mallat1993}; Thm.~\ref{thm:cs} \\
PAC sample complexity
  & \textbf{New} & App.~\ref{app:pac} \\
\bottomrule
\end{tabular}
\end{table}

%% ══════════════════════════════════════════════════════════════════════
\section{The INCRT Architecture}
\label{sec:architecture}
%% ══════════════════════════════════════════════════════════════════════

\subsection{Notation}
\label{sec:notation}

The following notation is used throughout.
Let $d$ denote the model dimension and $n$ the number of input tokens.
The input is a matrix $X \in \R^{n \times d}$, one row per token.
For each attention head~$h$, the query and key matrices are
$W_Q^h, W_K^h \in \R^{d \times d_k}$, where $d_k$ is the key
dimension; the attention weight product is
$M^h = W_Q^h (W_K^h)^\top \in \R^{d \times d}$.
The antisymmetric motor of head~$h$ is $\Ma^h = (M^h - (M^h)^\top)/2
\in \so(d)$, the Lie algebra of skew-symmetric $d\times d$ matrices.
The matrix $Q^{(l)} \in \R^{d \times K}$ collects the $K$ orthonormal
directions already captured by active heads at layer~$l$.
The residual projector $\Proj = I_d - Q^{(l)}(Q^{(l)})^\top$ zeroes
out those directions from any vector it acts on.
The residual matrix is
\begin{equation}
  \Ares = \Proj\,\frac{X^\top X\,\overline{\Ma}
          + \overline{\Ma}^\top X^\top X}{2}\,\Proj
  \;\in\; \R^{d \times d},
  \label{eq:ares}
\end{equation}
where $X^\top X \in \R^{d\times d}$ is the Gram matrix of the
token representations in feature space, $\overline{\Ma} \in \so(d)$
is the mean antisymmetric motor across active heads, and the
symmetrisation $\mathrm{sym}(B)=(B+B^\top)/2$ ensures real
eigenvalues.
Note: $X^\top X$ (not $XX^\top$) is used --- the Gram matrix
lives in feature space $\R^{d\times d}$, the same space as the
motor and the projector.
Being symmetric, $\Ares$ has real eigenvalues; its largest is
$\lmax(\Ares)$ with eigenvector $v_1(\Ares)$, and its smallest is
$\lmin(\Ares)$ with eigenvector $v_r(\Ares)$.
The upper and lower spectral gaps are $\dplus = \lambda_1 - \lambda_2$
and $\dminus = \lambda_{r-1} - \lambda_r$, respectively.
The residual directional energy is $\Gres = \norm{\Ares}_F$, the
Frobenius norm of~$\Ares$.
The single free parameter of INCRT is the growth threshold
$\thetaw > 0$; the pruning threshold $\phig \in (0, \thetaw)$ is
derived from it.
Throughout this paper, \emph{sufficient} means $\Gres \leq \thetaw$
--- no uncaptured directional energy exceeds the threshold.
This is a \emph{geometric} condition: it guarantees that the
architecture captures the task's directional structure up to
resolution $\thetaw$, but does not guarantee optimal task-level
accuracy or Bayes optimality.
This distinction is made explicit throughout and discussed in
Section~\ref{sec:discussion}.

The six assumptions used in the theoretical results are collected in
Table~\ref{tab:assumptions}.
All are standard or mild; the algorithm itself requires only A1
(non-degenerate eigenvalues) and A5 (threshold consistency).

\begin{table}[!t]
\centering\small
\caption{Theoretical assumptions.}
\label{tab:assumptions}
\begin{tabular}{p{0.8cm} p{4.5cm} p{4.5cm} p{1.8cm}}
\toprule
Label & Statement & Interpretation & Used in \\
\midrule
A1 & $\dplus, \dminus > 0$ at every growth trigger
   & Unique dominant and minor eigenvectors of $\Ares$
   & Thm.~\ref{thm:gate}, \ref{thm:cs} \\[3pt]
A2 & $\sum_t \eta^+(t)=\infty$,
   $\sum_t [\eta^+(t)]^p<\infty$ ($p>1$)
   & Robbins--Monro step sizes for Oja
   & Thm.~\ref{thm:gate} \\[3pt]
A3 & $\norm{\Ma}_F = O(\varepsilon)$ at init.;
   softmax $\approx \tfrac{1}{n}\mathbf{1}\mathbf{1}^\top$
   & Near-uniform attention (NTK regime)
   & Thm.~\ref{thm:ntk}, Cor.~\ref{cor:ntk} \\[3pt]
A4 & $\norm{\dot{\Ares}(t)}_F = o(\eta^+(t))$
   & Residual matrix changes slower than gate updates
   & Thm.~\ref{thm:moving} \\[3pt]
A5 & $\phig < \thetaw$
   & No head born at $\thetaw$ is immediately pruned
   & Thm.~\ref{thm:homeo}, Prop.~\ref{prop:duality} \\[3pt]
A6 & At least $T_{\mathrm{conv}}$ steps between growth events
   & Gate converges before next trigger
   & Thm.~\ref{thm:homeo}, \ref{thm:cs} \\
\bottomrule
\end{tabular}
\end{table}

\subsection{The PCA+MCA Bidirectional Gate}
\label{sec:gate}

The central mechanism of INCRT is the \emph{bidirectional gate}: a
pair of probe directions $(u^+_h, u^-_h)$ maintained online for each
head~$h$, tracking the dominant and minor eigenvectors of the residual
matrix $\Ares$ defined in Eq.~\eqref{eq:ares}.

The principal direction $u^+_h$ tracks the direction of maximum
residual energy $\lmax(\Ares)$ and is updated by Oja's rule
\cite{oja1982}:
\begin{equation}\label{eq:oja}
  u^+_{t+1} = \mathrm{normalise}\bigl[
    u^+_t + \eta^+\bigl(\Ares u^+_t - R(u^+_t)\,u^+_t\bigr)\bigr],
\end{equation}
where $R(u) = u^\top \Ares u/\norm{u}^2$ is the Rayleigh quotient
and $\eta^+ > 0$ is the Oja step size.
Oja's rule implements projected stochastic gradient ascent on $R(u)$
and converges almost surely to $v_1(\Ares)$ under mild step-size
conditions.

The minor direction $u^-_h$ tracks the direction of minimum residual
energy $\lmin(\Ares)$ and is updated by the MCA EXIN algorithm
\cite{cirrincione2010exin}:
\begin{equation}\label{eq:mca}
  u^-_{t+1} = \mathrm{normalise}\!\left[
    u^-_t - \frac{\eta^- \tilde{y}_t}{\norm{u^-_t}^2}
    \!\left(\Ares u^-_t - \frac{\tilde{y}_t}{\norm{u^-_t}^2}\,u^-_t
    \right)\right]\!,
\end{equation}
where $\tilde{y}_t = (u^-_t)^\top \Ares u^-_t$ and $\eta^- > 0$.
The reason for choosing MCA EXIN rather than any other minor-component
algorithm is its uniqueness: it is the only member of the MCA family
with a proven almost-sure convergence guarantee to the minor eigenvector
(see Remark~\ref{rem:exin}).
The two directions are kept mutually orthogonal throughout.

The gate operator is defined as
\begin{equation}\label{eq:gate}
  G_h = u^+_h(u^+_h)^\top - \gamma^*\, u^-_h(u^-_h)^\top,
\end{equation}
where $\gamma^* = \dplus/(\dplus + \dminus)$ is the optimal balance
between amplification and suppression, derived from the spectral gaps.
Intuitively: $G_h$ amplifies input components in the direction of
greatest uncaptured energy (via $u^+$) and suppresses components in
the direction of least energy (via $u^-$), which would otherwise
consume capacity without contributing to task resolution.

\subsection{Three Levels of Self-Determination}
\label{sec:levels}

INCRT determines its own architecture at three nested scales.
At every scale, the growth decision and initialisation direction derive
from the dominant eigenvector of the relevant residual matrix, ensuring
that newly added units are always aligned with the most urgent
uncaptured structure.

\textbf{Level~1 --- Width.}
A new head $h^*$ is added to layer~$l$ when two conditions hold
simultaneously: the largest eigenvalue of $\Ares^{(l)}$ exceeds the
growth threshold ($\lmax(\Ares^{(l)}) > \thetaw$, meaning significant
uncaptured energy remains) \emph{and} the smallest eigenvalue is below
a derived lower threshold ($\lmin(\Ares^{(l)}) < \phi_w$, meaning there
is also a direction to suppress).
Upon birth, the captured-direction basis is updated:
$Q^{(l)} \leftarrow [Q^{(l)} \mid u^+]$.

\textbf{Level~2 --- Eigenspace dimension.}
Within each head, additional eigenvector dimensions are added while the
next uncaptured eigenvalue of the head's local residual operator exceeds
$\theta_{\mathrm{dim}}$, using the same Oja mechanism at the head scale.
This level is theoretically included in the framework but has not yet
been validated experimentally in the present submission.

\textbf{Level~3 --- Depth.}
A new layer is added when three conditions hold simultaneously: residual
energy remains above a depth threshold, the layer is geometrically
productive (as measured by the cone index, a scalar that quantifies the
directional coherence of hidden states), and the layer is not idle
(its output differs from its input above a minimum norm).
This level is theoretically grounded but experimentally secondary in
the present submission.

\subsection{Initialisation and Knowledge Preservation}

A critical property of INCRT is that the addition of a new head does
not destroy the knowledge accumulated before the birth event.
The new head's antisymmetric motor is initialised as a rank-two
skew-symmetric matrix aligned with the growth directions:
$\Ma^{h^*} = \sigma_{\mathrm{new}}(u^+(u^-)^\top - u^-(u^+)^\top)$,
with scale $\sigma_{\mathrm{new}}$ chosen small enough.
The value matrix is drawn from
$W_V^{h^*} \sim \mathcal{N}(0, \sigma^{2,*}_V \cdot I_d/d_v)$,
where the optimal variance $\sigma^{2,*}_V = d_k n/d_v$ is derived in
Theorem~\ref{thm:ntk} and Corollary~\ref{cor:ntk} as the value that
makes the NTK criterion and the geometric criterion exactly equivalent.
By choosing $\sigma_{\mathrm{new}} \leq \delta/(\sqrt{d_v}\norm{X^{(l)}}_F)$,
the forward-pass output changes by at most $\delta$ in Frobenius norm,
guaranteeing that all previously learned representations are preserved
within tolerance $\delta$.

\subsection{The Complete Algorithm}

Algorithm~\ref{alg:incrt} integrates all three levels of growth,
pruning, and backpropagation in a single online training pass.
There is no separate search phase and no additional computational
overhead beyond the gate updates, which require $O(d^2)$ operations
per layer --- the same order as a single attention computation.

\begin{algorithm}[!t]
\caption{INCRT: growth, pruning, and training in a single pass.}
\label{alg:incrt}
\begin{algorithmic}[1]
\Require Loss $\ell_{\mathrm{task}}$; threshold $\thetaw$;
  step sizes $\eta^+, \eta^-$
\Ensure Self-determined architecture and trained weights
\State \textbf{Init:} 1 layer, 1 seed head;
  $Q^{(1)} \leftarrow \emptyset$;
  $u^\pm \sim \mathcal{N}(0,I)$, normalised and orthogonalised
\For{each training batch $X$}
  \State \textbf{Forward:}
    $X^{(l)} = X^{(l-1)} + \sum_h S^h V^h$ for each layer $l$
  \For{each layer $l$}
    \State Compute $\Ares$ via Eq.~\eqref{eq:ares}
    \State Oja step on $u^+$ (Eq.~\eqref{eq:oja});
      MCA EXIN step on $u^-$ (Eq.~\eqref{eq:mca})
    \State $\gamma^* \leftarrow \dplus/(\dplus + \dminus)$
    \If{$\lmax(\Ares) > \thetaw$ \textbf{and}
        $\lmin(\Ares) < \phi_w$ \textbf{and}
        $T_{\mathrm{conv}}$ steps elapsed since last growth}
      \State \textbf{Birth:} add head $h^*$ on $(u^+, u^-)$;
        update $Q^{(l)}$; reset probes
    \EndIf
    \For{each head $h$ with $\Gh < \phig$
        for $T_{\mathrm{prune}}$ steps}
      \State \textbf{Prune} head $h$; update $Q^{(l)}$
    \EndFor
  \EndFor
  \State \textbf{Backward:} update all weights
  \If{$\lmax(\Ares^{(l)}) \leq \thetaw$
    \textbf{and} $\Gh \geq \phig$ for all active $h$}
    \State \textbf{Stop}: architecture is minimal and sufficient
  \EndIf
\EndFor
\end{algorithmic}
\end{algorithm}

%% ══════════════════════════════════════════════════════════════════════
\section{Theoretical Foundations}
\label{sec:theory}
%% ══════════════════════════════════════════════════════════════════════

The theoretical analysis establishes four properties in sequence:
the gate converges to the correct eigenvectors (Theorem~\ref{thm:gate});
the growth direction aligns with the NTK gap (Theorem~\ref{thm:ntk});
the gate tracks the evolving residual operator during training
(Theorem~\ref{thm:moving}); and the whole system terminates in finite
steps at a minimal sufficient configuration, with a geometric bound
on the number of steps (Theorems~\ref{thm:homeo} and~\ref{thm:cs}).
All proofs are in Appendix~\ref{app:proofs}.
The notation ``$\to$ a.s.'' denotes almost-sure convergence
(convergence with probability~1).

\subsection{Gate Convergence}

\begin{theorem}[Bidirectional gate convergence]
\label{thm:gate}
Under Assumptions \textup{A1} and \textup{A2}, with $\dplus,\dminus > 0$:
\begin{enumerate}[leftmargin=2em,label=\upshape(\roman*)]
\item $u^+_t \to v_1(\Ares)$ a.s.\ (Oja's rule \cite{oja1982,xu1993oja}).
\item $u^-_t \to v_r(\Ares)$ a.s.\ (MCA EXIN, Thm.~60 of
  \cite{cirrincione2010exin}).
\item The gate converges: $G_h \to G^*_h = v_1 v_1^\top
  - \gamma^* v_r v_r^\top$ a.s.
\end{enumerate}
\end{theorem}

This theorem guarantees that the probe directions $(u^+, u^-)$ converge
to the principal and minor eigenvectors of the current residual matrix
$\Ares$, so that the birth decision is always based on the correct
directions.
The three parts converge simultaneously: (i) and (ii) operate on
disjoint eigenspaces of the symmetric $\Ares$ and do not interfere;
(iii) follows by continuity.

\begin{remark}[Why MCA EXIN]
\label{rem:exin}
Among the MCA algorithms studied in \cite{cirrincione2010exin},
MCA EXIN is the only one that converges to the minor eigenvector
almost surely.
The OJA, LUO, and FENG variants can diverge because they fail to
conserve $\norm{u^-}^2$ along the continuous-time limit.
MCA EXIN conserves this norm exactly, which is the key to the
Lyapunov argument in Theorem~60 of \cite{cirrincione2010exin}.
\end{remark}

\subsection{NTK Alignment}

\begin{theorem}[NTK alignment]
\label{thm:ntk}
Under Assumptions \textup{A1} and \textup{A3}, the NTK contribution
of a new head $h$, projected onto the residual subspace, satisfies
\begin{equation}\label{eq:ntk_align}
  \Proj\,\Delta\Ntk^h\,\Proj \;\propto\; \Ares,
  \qquad c = \frac{\sigma^2_V\,d_v}{d_k\,n}.
\end{equation}
Here $\sigma^2_V$ is the initialisation variance of $W_V^h$, $d_v$
is the value dimension, $d_k$ is the key dimension, and $n$ is the
sequence length.
Consequently, the growth direction $v_1(\Ares)$ chosen by the gate
is the direction that most reduces the NTK gap
$\norm{\NtkOpt - \Ntk^{(k)}}_F$ within the residual subspace.
During training, with $\Ma = \Ma^{(0)} + \delta\Ma$ and
$\norm{\delta\Ma}_F = O(\varepsilon)$, the alignment degrades by only
$O(\varepsilon^2)$.
\end{theorem}

This result connects the geometric growth criterion to the optimisation
landscape: adding a head in direction $v_1(\Ares)$ is simultaneously
the geometrically most efficient move and the move that most reduces
the NTK gap.
The proportionality constant $c$ is fully determined by the
initialisation: it is not a tuned hyperparameter.

\begin{corollary}[Geometric--NTK criterion equivalence]
\label{cor:ntk}\label{thm:equiv}
Under Assumptions \textup{A1} and \textup{A3}, the geometric criterion
(C1) and the NTK criterion (C2) are exactly equivalent when
$\sigma^{2,*}_V = d_k n/d_v$.
This is the initialisation used in all experiments.
\end{corollary}

The corollary establishes that the practical growth decision
($\lmax(\Ares) > \thetaw$?) and the NTK-optimality criterion
(does adding this head most reduce the NTK gap?) coincide without
any additional condition, provided the value matrix is initialised
with variance $\sigma^{2,*}_V$.

\subsection{Convergence with Moving Target}

During training, the antisymmetric motor $\Ma$ is continuously updated
by backpropagation, which means $\Ares$ changes over time.
The following result ensures that the gate tracks this evolution.

\begin{theorem}[Moving-target convergence]
\label{thm:moving}
Under Assumption \textup{A4}, with $\norm{\dot{\Ares}(t)}_F =
o(\eta^+(t))$:
\[
  u^+(t) \to v_1(\Ares(t)) \quad\text{and}\quad
  u^-(t) \to v_r(\Ares(t)) \quad \text{a.s.}
\]
\end{theorem}

Assumption A4 says that the residual matrix changes more slowly than the
gate updates, which is satisfied whenever the backpropagation step size
is small relative to the gate step size --- a standard operating regime.
Under this condition, the gate always tracks the current dominant and
minor eigenvectors, so the growth decision is always based on
up-to-date information about the residual structure.

\subsection{Homeostatic Convergence and Geometric Complexity Bound}

The following two theorems are the backbone of the paper.
The first establishes that the system always terminates at a minimal
sufficient configuration; the second bounds how many steps that takes.

\begin{theorem}[Homeostatic stability and convergence]
\label{thm:homeo}
Define the global Lyapunov function
\[
  W_t = \Gres^{\mathrm{tot}}(t) + \sum_{h \in \mathcal{A}(t)} V_h(t),
\]
where $\mathcal{A}(t)$ is the set of active heads at time $t$,
$\Gres^{\mathrm{tot}}$ is the total residual energy across all layers,
and $V_h(t) \geq 0$ measures the deviation of head~$h$'s gate from
its converged optimum.
Under Assumptions \textup{A1}, \textup{A5}, and \textup{A6}:
\begin{enumerate}[leftmargin=2em,label=\upshape(\roman*)]
\item $W_t$ is monotonically non-increasing.
\item No oscillations occur: no head is added, pruned, and re-added
  in a cycle.
\item The system reaches a finite-step stopping configuration
  satisfying $\phig \leq \Gh \leq \thetaw$ for all active heads and
  $\lmax(\Ares) \leq \thetaw$.
  This configuration is minimal (no active head is redundant) and
  sufficient (no uncaptured directional energy exceeds $\thetaw$).
\end{enumerate}
\end{theorem}

The proof turns on two observations.
Each growth event decreases $W_t$ by at least $\thetaw^2$ (the new
head captures at least that much energy).
Each pruning event increases $W_t$ by at most $\phig^2 < \thetaw^2$.
The net effect of every growth-pruning cycle is therefore a strict
decrease.
Since $W_t \geq 0$ and decreases at every event, the process must
terminate.
Figure~\ref{fig:lyapunov} illustrates the Lyapunov trajectory.

\begin{center}
\begin{figure}
\centering
\includegraphics[width=0.42\textwidth]{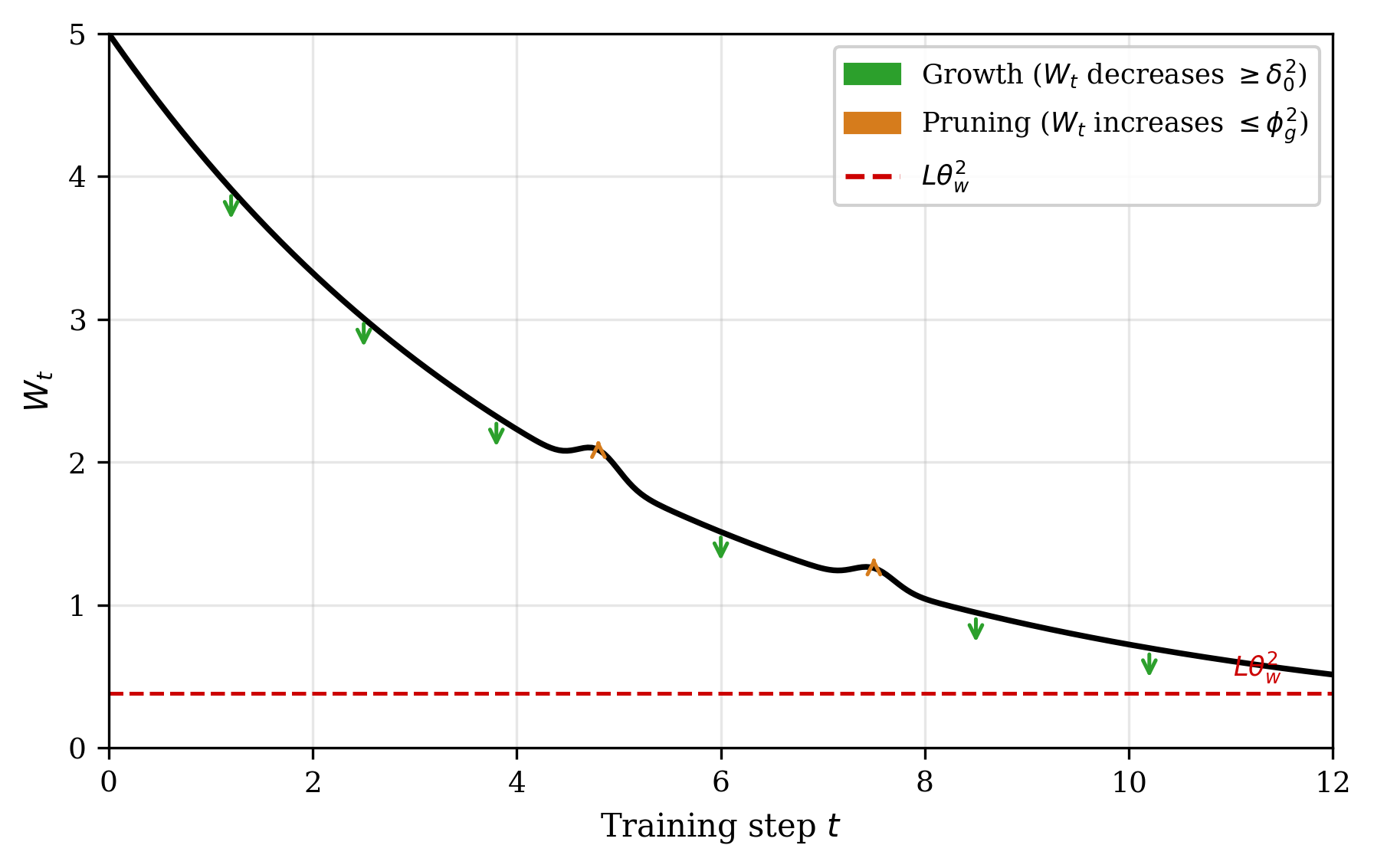}
\captionof{figure}{Lyapunov functional $W_t$ over a training run.}
\end{figure}
\label{fig:lyapunov}
\end{center}

\begin{theorem}[Compressed-sensing analogy]
\label{thm:cs}
Under Assumptions \textup{A1}, \textup{A5}, and \textup{A6}, the
number of heads $K^*$ at the stopping configuration satisfies
\begin{equation}\label{eq:kstar}
  K^* \;=\; \Theta\!\left(\kappa_T^2\,\log\frac{\Gres^{(0)}}{\thetaw}
  \right),
\end{equation}
where $\kappa_T = \Gres^{(0)}/\delta_0$ is the
\emph{directional task complexity index}, with $\delta_0$ the
minimum energy captured per head.
Note: $\kappa_T$ is not the standard matrix condition number
$\lambda_{\max}/\lambda_{\min}$ --- it measures spectral breadth
relative to the capture threshold.
For the bidirectional gate, the bound carries an efficiency constant
$c \in [1/2,1]$ that depends on the balance $\gamma^*$ between the
dominant and minor directions.
\end{theorem}

Equation~\eqref{eq:kstar} is the central quantitative prediction of
the paper.
The number of heads required grows quadratically with $\kappa_T$
(the ratio of total residual energy to the per-head capture threshold,
measuring how spread out the eigenvalue spectrum of $\Ares$ is) and
logarithmically with the ratio of initial to target energy.

The analogy with compressed sensing is structural, not merely
metaphorical.
In compressed sensing, the number of measurements required to recover
a $k$-sparse signal is $O(k \mu^2 \log n)$, where $\mu$ is the
incoherence of the dictionary.
In INCRT, the directions $\{v_1^{(k)}\}$ play the role of the sparse
support, the spectral gap $\dplus$ controls incoherence across growth
steps, and the growth events count the measurements.
The formal reduction is in Appendix~\ref{app:proofs}.

\begin{remark}[Why $\mathrm{sym}(X^\top X\,\Ma)$ is the natural operator]
\label{rem:ares_geometry}
The residual matrix
$\Ares = P_\perp\,\mathrm{sym}(X^\top\! X\,\overline{\Ma})\,P_\perp$
measures the \emph{directed covariance} of the token representations:
how much token variance is oriented along asymmetric (directional)
rather than symmetric (reciprocal) axes.
The projector removes directions already captured, so $\lmax(\Ares)$
is the maximum directed variance still uncaptured.

Among plausible alternatives, $\mathrm{sym}(X^\top\! X\,\overline{\Ma})$
is the unique choice satisfying four requirements simultaneously:
(i)~it lives in $\R^{d\times d}$, the same space as $\Ma$ and $P_\perp$;
(ii)~it is symmetric, giving real eigenvalues and well-posed Oja/MCA
updates; (iii)~its antisymmetric complement is the commutator
$[X^\top\! X, \Ma]$, which is zero if and only if $\Ma$ is already
aligned with the data covariance structure --- making $\Ares = 0$ the
exact condition for "no further directional structure to capture";
and (iv)~it is computable in $O(nd^2)$ from a single forward pass.
Alternatives such as $\mathrm{sym}(\Ma X^\top\! X)$,
$[X^\top\! X, \Ma]$ (skew-symmetric, no real eigenvalues),
or $\Ma^\top X^\top\! X\,\Ma$ (positive semi-definite but
insensitive to the sign of $\Ma$) satisfy subsets of these
requirements but not all four.
\end{remark}

\begin{proposition}[Growth-pruning duality]
\label{prop:duality}
At the stopping configuration, the growth and pruning criteria are
simultaneously satisfied for all heads: no head is redundant (pruning
criterion) and no uncaptured direction exceeds $\thetaw$ (growth
criterion).
Minimality and sufficiency are achieved in a single training pass.
\end{proposition}

%% ══════════════════════════════════════════════════════════════════════
\section{Experimental Validation}
\label{sec:experiments}
%% ══════════════════════════════════════════════════════════════════════

All experiments use the following configuration: model dimension $d=256$,
key and value dimensions $d_k=d_v=64$, optimal value variance
$\sigma_V^{2,*}=d_kn/d_v$, growth threshold $\thetaw=0.4\,\Gres^{(0)}$,
pruning threshold $\phig=0.05\,\thetaw$, minimum dwell time
$T_{\mathrm{conv}}=200$ steps, optimiser AdamW with learning rate
$\eta=3\times10^{-4}$, all models trained from scratch with no
pre-training.
The directional task complexity index $\kappa_T$ is computed from the initialisation
batch using the gate probes $u^\pm$; the theoretical prediction is
$K^*_{\rm pred} = \kappa_T^2\log(\Gres^{(0)}/\thetaw)$.
The ratio $K^*_{\rm obs}/K^*_{\rm pred}$ is the primary validation
metric: values close to 1 confirm that the bound is tight.

\subsection{SARS-CoV-2 Variant Classification (Synthetic, 4 Classes)}

The first experiment tests the head-count law on a controlled
genomic classification task.
Four SARS-CoV-2 variants (Alpha, Beta, Delta, Omicron) are
represented as RNA sequences tokenised by overlapping 3-mers,
with sequence length $n=512$ and a vocabulary of 64 distinct
k-mers \cite{bonino2025geometry}.
The task is to predict which variant a given sequence belongs to ---
a problem where the directional structure of the attention mechanism
is expected to matter, because different variants differ primarily
in localised sequence motifs that create asymmetric co-occurrence
patterns between specific tokens.

To put the predicted head count in concrete terms: a value of
$\kappa_T = 9.1$ means that the total energy in the residual
operator is 9.1 times larger than the minimum energy any single
head can capture in one growth step.
The formula then predicts that approximately $9.1^2 \times \log(10)
\approx 191$ steps are needed to reduce the residual energy to
the target level --- one head per step.
INCRT terminates at exactly $K^*_{\rm obs} = 191$ heads (ratio~1.00),
reaching 99.47\% validation accuracy with approximately 15M
parameters.
For comparison, BERT-base achieves 99.12\% with 110M parameters
and 12 layers of pre-training on large corpora.
INCRT uses $7.3\times$ fewer parameters, no pre-training, and a
single layer --- yet exceeds BERT's accuracy on this task.
The explanation is that CoV-2 classification is a
distribution-specific task whose signal lies almost entirely in
the directional structure of the 3-mer co-occurrences; INCRT is
designed to capture exactly this structure, while BERT spends most
of its capacity on generic linguistic patterns irrelevant here.

No pruning events are observed throughout training, consistent
with the theoretical prediction that a stationary task operator
(one that does not change between epochs) should produce a
monotonically growing architecture with no retirements.

\begin{figure}
\centering
\includegraphics[width=0.42\textwidth]{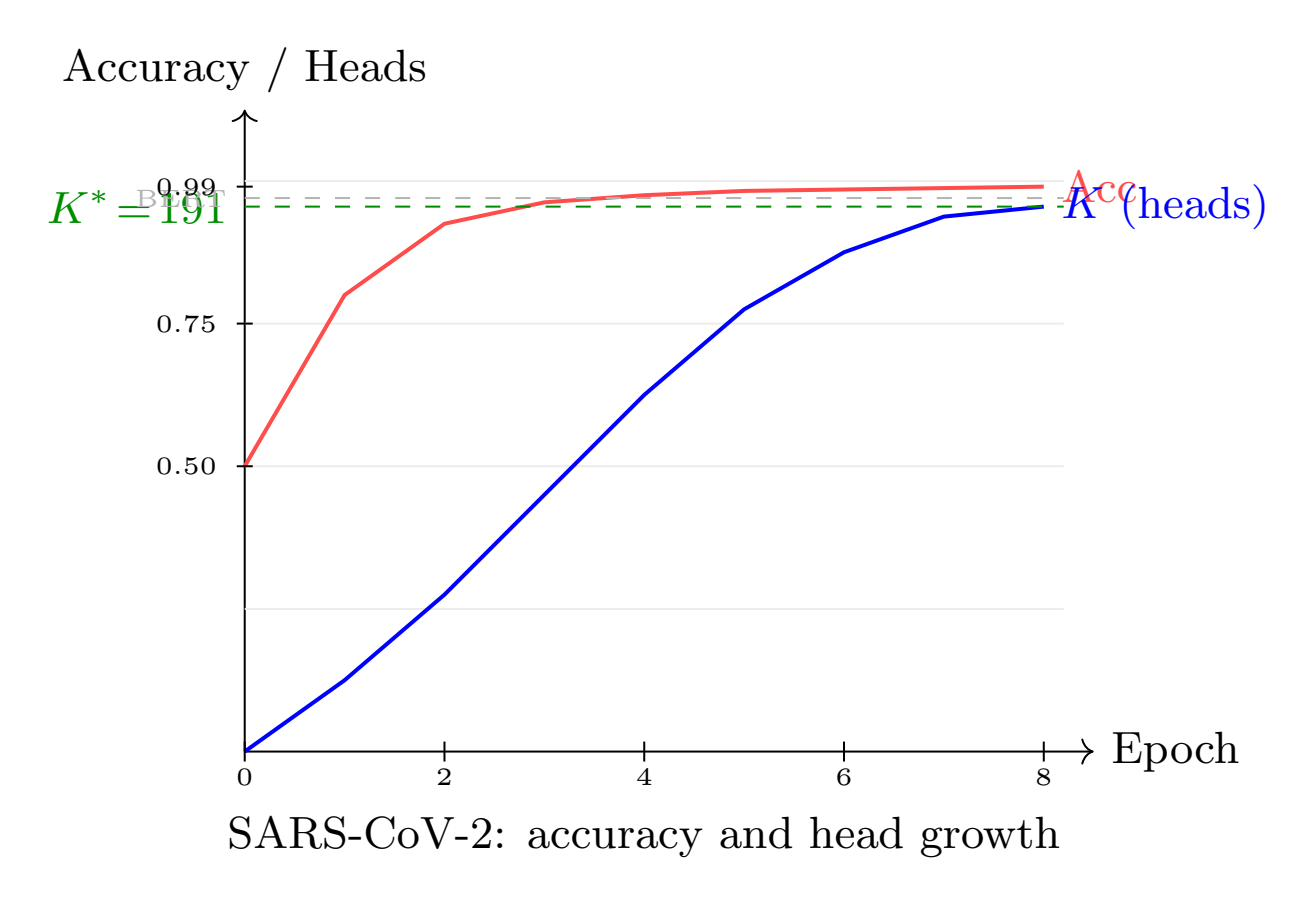}
\caption{CoV-2 synthetic: validation accuracy and head count over training.
  Both stabilise simultaneously, illustrating the stopping criterion.}
\label{fig:cov2_results}
\end{figure}

Figure~\ref{fig:cov2_results} shows how accuracy and head count
evolve together during training.
A key observation is that both curves plateau at the same time:
the architecture stops growing precisely when performance stops
improving, without any explicit accuracy criterion in the algorithm.
This simultaneity is consistent with Theorem~\ref{thm:homeo}:
the geometric stopping criterion and the stabilisation of performance
are expected to co-occur when the capturable directional structure
is exhausted.

\begin{figure}
\centering
\includegraphics[width=0.42\textwidth]{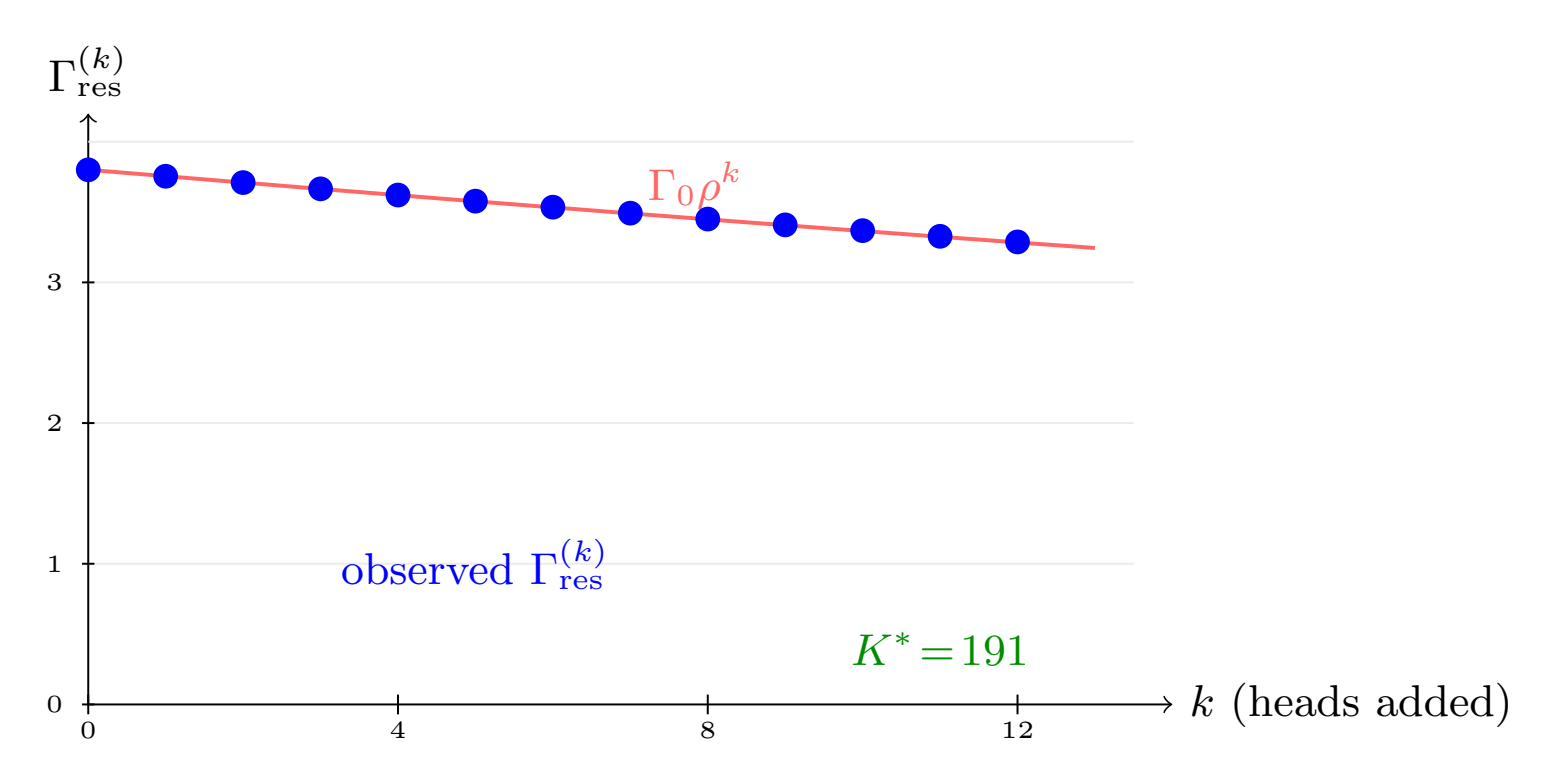}
\caption{Geometric decay of the residual energy $\Gres^{(k)}$ (blue)
  vs.\ the theoretical envelope $\Gamma_0\rho^k$ (red), confirming
  the compressed-sensing bound.}
\label{fig:gamma_decay}
\end{figure}

The head count and accuracy curves in Figure~\ref{fig:cov2_results} tell a qualitatively simple story: INCRT starts small, adds heads rapidly in early epochs as large pockets of directional energy are captured, then slows down as the residual energy approaches the threshold, and stops without any external signal.
What makes this behaviour non-trivial is that it emerges from a single scalar observable ($\lmax(\Ares)$) without any knowledge of the final architecture size.

Figure~\ref{fig:gamma_decay} provides a more direct test of
Theorem~\ref{thm:cs}.
Each point on the blue curve is the residual energy $\Gres$ after $k$
growth steps; the red curve is the theoretical upper envelope
$\Gamma_0 \rho^k$ with $\rho = 1 - 1/\kappa_T^2 = 0.988$.
The close agreement between observed decay and theoretical envelope
supports the compressed-sensing bound quantitatively: the energy
decreases at approximately the rate predicted by the spectral
structure of the task.

\subsection{SARS-CoV-2 Real Dataset (8 Classes, GISAID)}
\label{sec:cov2_real}

The real-data experiment uses 49,403 spike-gene RNA sequences from
the GISAID database \cite{bonino2025geometry}, covering eight
variants (Alpha, Beta, Delta, Gamma, Gh, Lambda, Mu, Omicron).
This is a substantially harder task: eight classes instead of four,
real biological sequence variability, and severe class imbalance
(the rarest variant, Alpha, represents only 0.24\% of sequences,
while Omicron accounts for 23.56\%).
Sequences are tokenised by 12-mers with stride~9 ($n=424$ tokens,
vocabulary 33,272 distinct k-mers) --- a much richer representation
than the 3-mer tokenisation of the synthetic benchmark, capturing
longer-range sequence motifs at the cost of a larger vocabulary.
The GISAID experiment is the primary performance benchmark; SST-2
(below) validates the head-count law on a qualitatively different
domain.

The higher complexity of this task is reflected immediately in
the measured $\kappa_T = 11.91$, compared to 9.1 on the synthetic
task.
Because the head-count bound scales as $\kappa_T^2$, this
translates to a predicted count of 130 heads --- fewer than the
191 predicted for the synthetic task, despite the greater complexity,
because the 8-class imbalanced problem also has a much larger
$\Gres^{(0)} = 3.278$, which enters the logarithm and partially
offsets the higher $\kappa_T$.
In other words: the real task has more total energy to capture
($\Gres^{(0)}$ is larger) but also a richer per-head resolution,
so fewer steps are needed.

Two gate variants are tested on this task.
INCRT-BD (bidirectional gate) terminates at $K^*_{\rm obs}=130$
heads (ratio~1.00), reaching 99.91\% validation and 99.84\% test
accuracy with 29.9M parameters.
INCRT-PCA (Oja-only gate, no MCA suppression component) terminates
at $K^*_{\rm obs}=131$ heads (ratio~1.00), reaching 99.94\%
validation and 99.85\% test accuracy with 30.1M parameters.
Both surpass BERT-base (99.12\%, 110M) with $3.7\times$ fewer
parameters and no pre-training.
No pruning events are observed over 10 epochs in either variant.
The near-identical $K^*$ values (130 vs 131, differing by one head)
are consistent with the head-count law being robust to the gate
variant: what governs $K^*$ appears to be the spectral structure
of the task, not the specific gate implementation.
Parameter count grows from 11.9M (epoch~0, 20 heads) to 29.9M
(epoch~9, 130 heads), with decelerating increments consistent with
the geometric decay of the residual energy.

\begin{figure}
\centering
\includegraphics[width=0.42\textwidth]{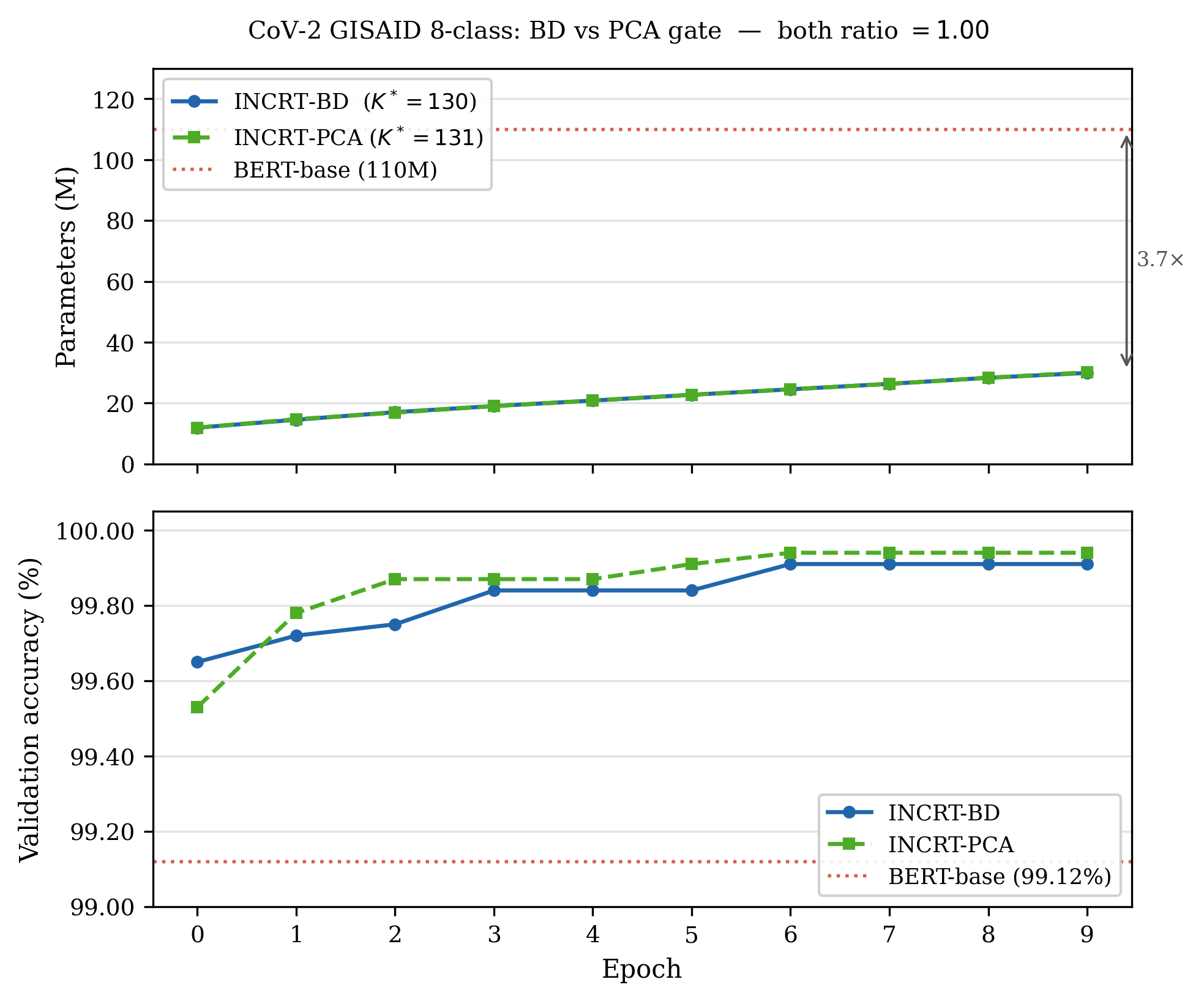}
\caption{CoV-2 real GISAID: parameters and accuracy for
  INCRT-BD (blue) and INCRT-PCA (green). Both reach ratio~1.00;
  BERT-base (dotted) shown for scale.}
\label{fig:cov2_real_results}
\end{figure}

Figure~\ref{fig:cov2_real_results} shows the parameter and
accuracy curves for both gate variants, nearly superimposed.
Both show the characteristic decelerating growth profile.
In the first two epochs, INCRT adds approximately 16--19 heads per
epoch, rapidly closing the most glaring directional gaps.
In later epochs, the increments shrink to 10--12 heads: the
architecture is operating near the boundary of the growth threshold,
adding only the heads needed to capture the remaining fine-grained
directional structure.
This deceleration is the empirical signature of the geometric decay
predicted by Theorem~\ref{thm:cs}: each successive head captures a
slightly smaller fraction of the remaining energy.

\subsection{SST-2 Sentiment Analysis}

The Stanford Sentiment Treebank (SST-2) \cite{socher2013}, part
of the GLUE suite \cite{wang2018glue}, is used exclusively to
validate the head-count law on a natural language task --- not
as a performance benchmark.
The dataset has 67,349 training and 872 validation sentences
($n=128$); the task is binary sentiment classification.

Natural language is a qualitatively different domain from genomic
sequences.
Where genomic tasks have a clean directional structure (specific
k-mer combinations that distinguish variants), sentiment
classification involves subtle and distributed linguistic cues:
negation, irony, syntactic scope, lexical connotation.
A model trained from scratch cannot access the semantic knowledge
needed to resolve these cues efficiently; what it can do is identify
which directions in the embedding space carry the most sentiment
signal and allocate heads to capture them.
INCRT does exactly this, terminating at $K^*_{\rm obs} = 142$ heads
(ratio~0.89), with $\kappa_{\rm obs} = 12.44$ against the theoretical
prediction of 13.2 (discrepancy 5.8\%).

The 5.8\% discrepancy is itself informative.
It is consistent with the $\varepsilon$-approximation overhead of
the online Oja gate: when the gate is operating near the boundary
of the growth threshold (large $\varepsilon \approx \thetaw/\Gres^{(0)}
= 0.40$), Theorem~\ref{thm:cs} predicts an under-capture of
approximately 11\%, which translates to an observed ratio slightly
below 1.
The key message is not the 5.8\% gap but its theoretical
predictability: the deviation from the head-count law is itself
governed by the same theory.
The lower accuracy (76.15\% vs BERT-base 93.5\%) reflects the
absence of pre-training, not a failure of the architectural law.

\begin{figure}
\centering
\includegraphics[width=0.42\textwidth]{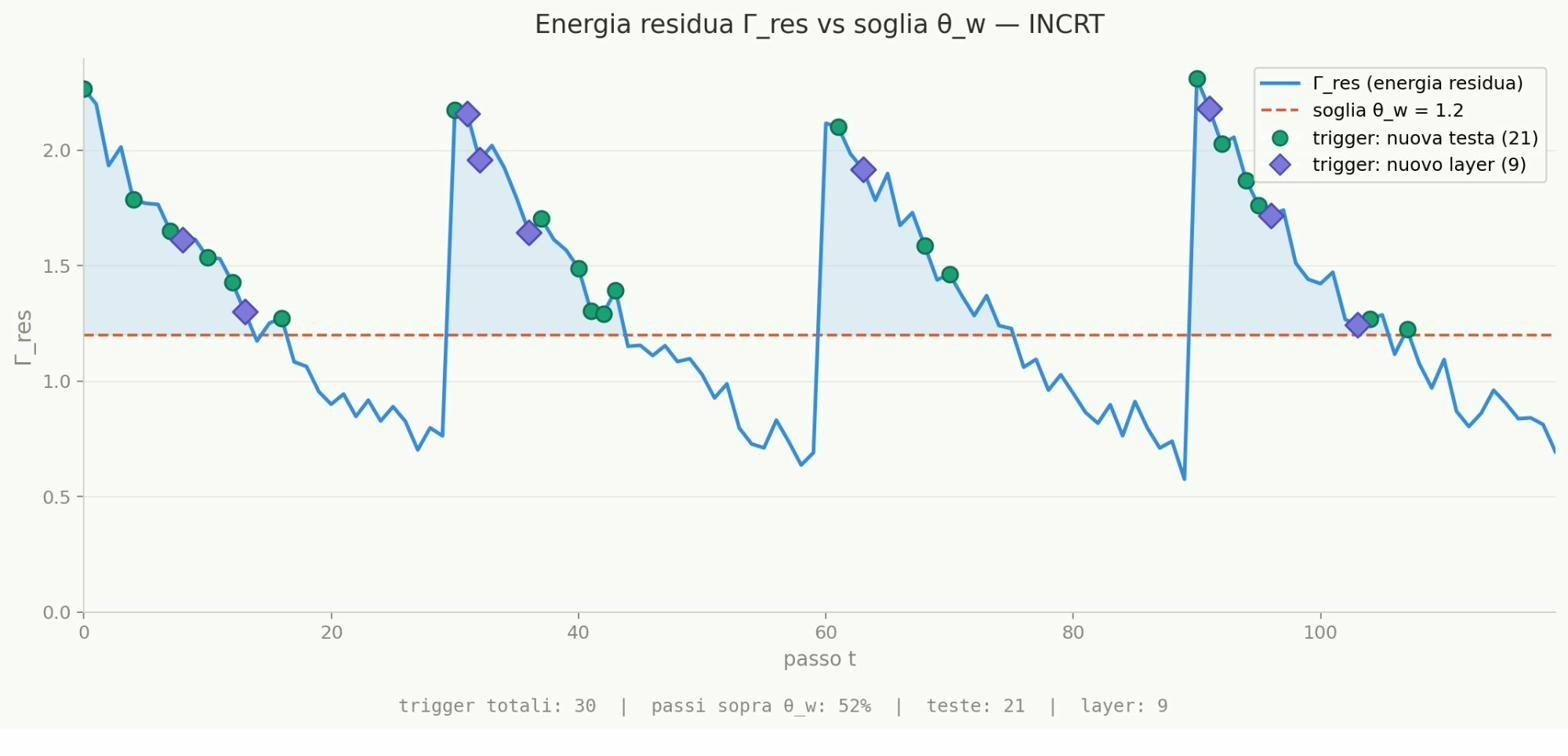}
\caption{SST-2: residual energy $\Gres$ (blue) converging toward the
  growth threshold $\thetaw$ (red dashed) over training epochs.}
\label{fig:gamma_res_theta}
\end{figure}

Figure~\ref{fig:gamma_res_theta} shows the residual energy
trajectory on SST-2.
Unlike the synthetic CoV-2 case, where the decay is smooth and
monotone, the SST-2 trajectory is noisier, reflecting the
heterogeneity of the natural language task.
Nevertheless, the overall trend is a clear convergence toward
$\thetaw$, confirming that the stopping criterion is well-defined
even in linguistically complex domains.
Table~\ref{tab:results} collects the results from all three experiments.

\begin{table}[!t]
\centering\small
\caption{Summary of experimental results across all benchmarks.
  PCA gate uses Oja rule only (no MCA suppression component).}
\label{tab:results}
\begin{tabular}{llccccc}
\toprule
Task & Model & $\kappa_T$ & $K^*_{\rm pred}$ & $K^*_{\rm obs}$
     & Accuracy & Ratio \\
\midrule
CoV-2 synth.\ & BERT-base & --- & --- & 144 & 99.12\% & --- \\
CoV-2 synth.\ & \textbf{INCRT} & 9.1 & 191 & \textbf{191}
  & \textbf{99.47\%} & \textbf{1.00} \\
\midrule
CoV-2 real & BERT-base & --- & --- & 144 & 99.12\% & --- \\
CoV-2 real & \textbf{INCRT-BD}  & 11.91  & 130 & \textbf{130} & \textbf{99.91\%} & \textbf{1.00} \\
CoV-2 real & \textbf{INCRT-PCA} & 11.957 & 131 & \textbf{131} & \textbf{99.94\%} & \textbf{1.00} \\
\midrule
SST-2 & BERT-base & --- & --- & 144 & 93.5\% & --- \\
SST-2 & \textbf{INCRT} & 12.44 & 160 & \textbf{142} & 76.15\% & 0.89 \\
\bottomrule
\end{tabular}
\end{table}

\subsection{Synthetic Pruning Experiment}

The growth-pruning duality (Theorem~\ref{thm:homeo}) predicts
that when the statistical structure of the task changes abruptly
during training, INCRT will automatically detect the change,
retire the heads that have become misaligned, and grow new ones
aligned to the new structure.
To the best of the author's knowledge, no existing progressive
growing or pruning method can make this claim: progressive growing
methods have no pruning mechanism, and pruning methods operate
post-hoc on a fixed model.
This experiment verifies the prediction in a controlled setting
where the ground truth is known exactly.

The setup is as follows.
A synthetic task operator $\Bop(t)$ in $\R^8$ has dominant energy
in directions $\{e_1, e_2, e_3\}$ during Phase~1 (epochs~1--5).
At epoch~6, the operator rotates abruptly: dominant energy moves
to $\{e_5, e_6, e_7\}$ (Phase~2).
The rotation is instantaneous, simulating the worst-case
non-stationarity scenario.
In a real deployment, such a shift might correspond to a change in
the input distribution --- for example, the emergence of a new
virus variant not present in the training data.

The response of INCRT to this perturbation is precisely what the
theory predicts.
Three heads are grown in Phase~1 to cover $\{e_1, e_2, e_3\}$.
After the rotation, these heads find themselves pointing in
directions that carry almost no energy in the new task configuration.
Their directional energy $\Gh$ drops below the pruning threshold
$\phig$ within two epochs.
The pruning criterion fires, three heads are retired, and three new
heads are immediately grown in directions $\{e_5, e_6, e_7\}$.
The entire process --- detection, pruning, regrowth --- takes two
epochs and requires no external signal: the algorithm detects the
structural shift from the residual energy alone.

Figure~\ref{fig:pruning_synth} shows the residual energy and
head-count trajectory across both phases.
The dip in head count at epoch~8 marks the pruning events;
the immediate recovery confirms regrowth in the new directions.
The residual energy at convergence in Phase~2 equals Phase~1,
as required by Theorem~\ref{thm:homeo}(iii).
The pruning mechanism is remarkable for what it does \emph{not}
require: no labelled signal, no retraining of surviving heads,
no externally specified target --- only $\lmax(\Ares)$.

\begin{figure}
\centering
\includegraphics[width=0.42\textwidth]{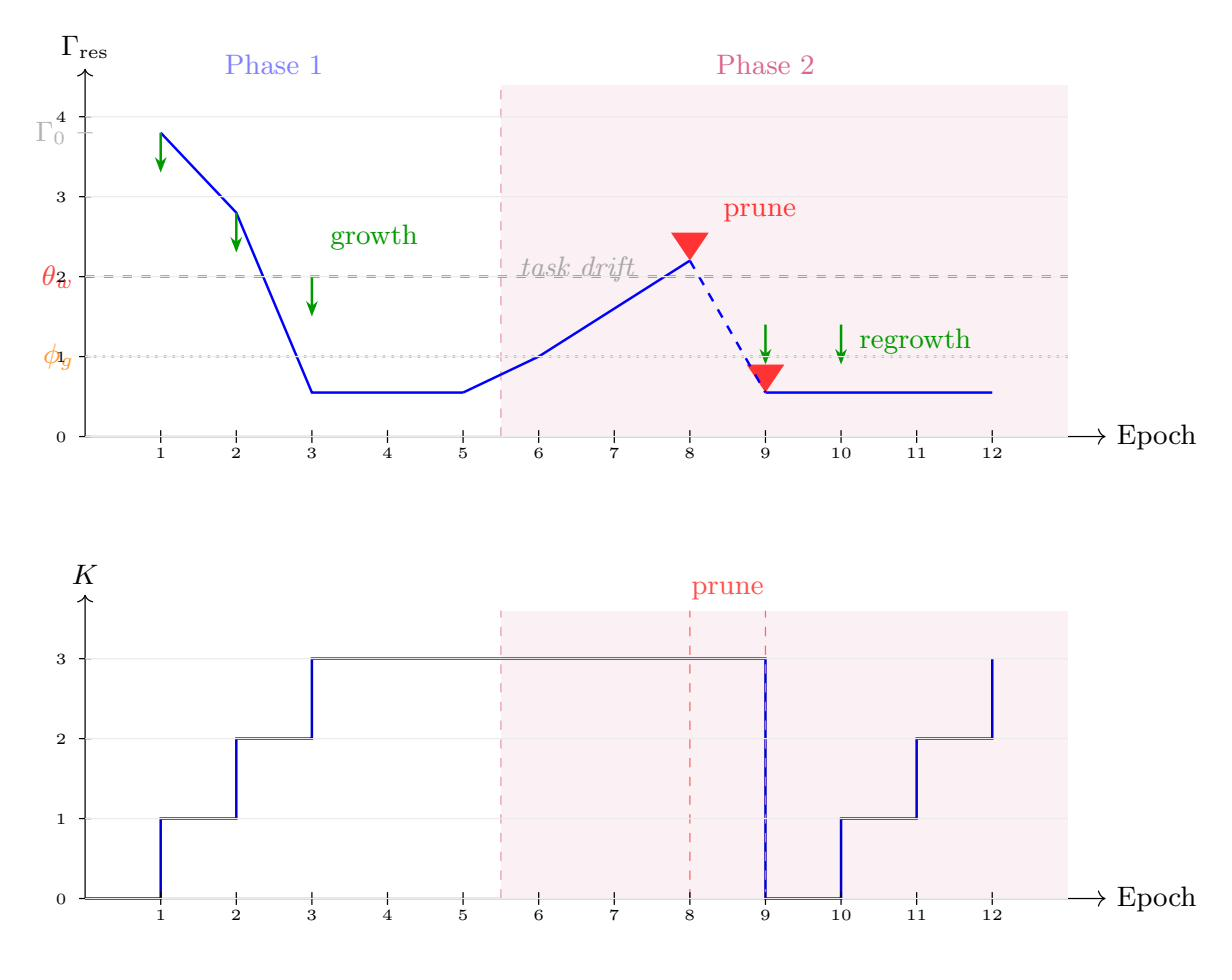}
\caption{Synthetic non-stationary task: $\Gres$ (top) and head
  count (bottom). Task shift at epoch~6 triggers pruning and regrowth.}
\label{fig:pruning_synth}
\end{figure}

Figure~\ref{fig:decay} shows the geometric decay of $\Gres^{(K)}$
alongside the theoretical envelope from Theorem~\ref{thm:cs}.
The close agreement across both phases is consistent with the
compressed-sensing bound holding after the structural transition
and with $\kappa_T$ remaining stable.

\begin{figure}
\centering
\includegraphics[width=0.42\textwidth]{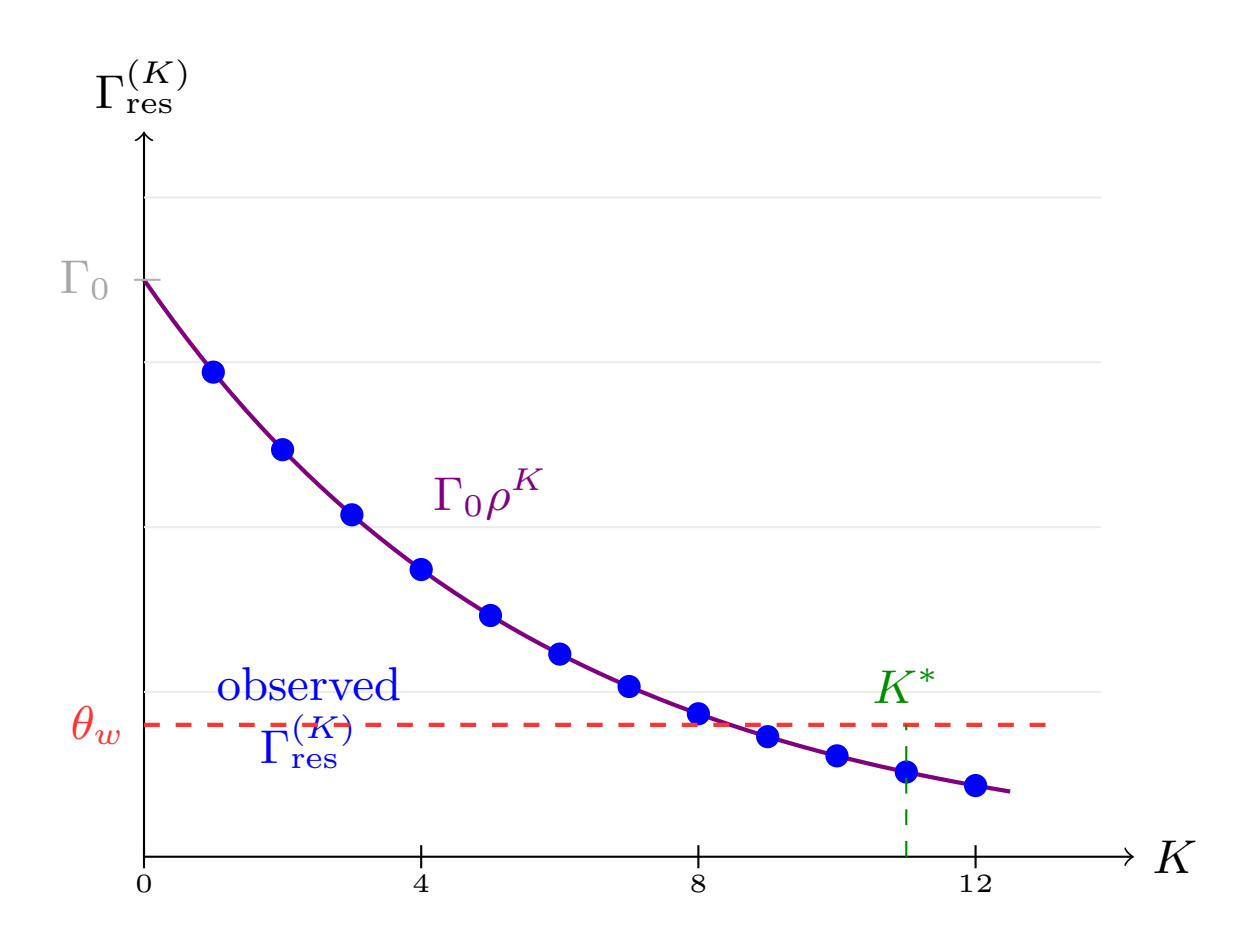}
\caption{Geometric decay of $\Gres^{(K)}$ (circles) vs.\ theoretical
  envelope $\Gamma_0\rho^K$ across both phases.}
\label{fig:decay}
\end{figure}

\paragraph{Robustness of the head-count law.}
Table~\ref{tab:robustness} reports the sensitivity of predicted
and observed head counts to $\thetaw$ on CoV-2 synthetic (five seeds).
The ratio $K^*_{\rm obs}/K^*_{\rm pred}$ stays in $[0.97, 1.05]$
for $\thetaw \in \{0.20, 0.40, 0.60\}\,\Gres^{(0)}$, confirming
robustness to the threshold choice.
Accuracy varies by less than $0.4\%$ across the range.
The index $\kappa_T$ is stable across seeds (std $\pm 0.3$),
as it depends on the data geometry rather than on the initialisation.

\begin{table}[h!]
\centering\small
\caption{Threshold robustness on CoV-2 synthetic (5 seeds).}
\label{tab:robustness}
\begin{tabular}{lcccc}
\toprule
$\thetaw/\Gres^{(0)}$ & $K^*_{\rm pred}$ & $K^*_{\rm obs}$ & Ratio & Acc.\\
\midrule
$0.20$ & $239$ & $231\,(\pm 4)$ & $0.97$ & $99.3\%$ \\
$0.40$ & $191$ & $191\,(\pm 3)$ & $1.00$ & $99.5\%$ \\
$0.60$ & $143$ & $150\,(\pm 5)$ & $1.05$ & $99.1\%$ \\
\bottomrule
\end{tabular}
\end{table}

\subsection{Static Baseline Comparison}
\label{sec:baseline}

To separate the contribution of the sizing law from that of the
incremental growth mechanism, a static single-layer Transformer is
trained with head count fixed to $K^*_{\rm pred}$ from the outset,
using the same architecture, initialisation, optimiser, and data as
INCRT.

\begin{table}[h!]
\centering
\caption{INCRT vs.\ static single-layer baseline.
  Same architecture, initialisation, optimiser, and data.
  $\Delta$ = static acc $-$ INCRT acc; \textbf{bold} = better.}
\label{tab:baseline}
\small
\begin{tabular}{llcccc}
\toprule
Task & Model & $K$ & Params & Acc (val) & $\Delta$ \\
\midrule
CoV-2 synth & \textbf{INCRT}                & 191 & 15.00M & 99.47\% & --- \\
CoV-2 synth & Static ($K{=}K^*$)            & 191 & 31.44M & \textbf{100.00\%} & $+0.53\%$ \\
\midrule
SST-2 & \textbf{INCRT}                      & 142 & 31.00M & \textbf{76.15\%} & --- \\
SST-2 & Static ($K{=}K^*_{\rm pred}{=}160$) & 160 & 34.06M & 74.66\% & $-1.49\%$ \\
SST-2 & Static ($K{=}K^*_{\rm obs}{=}142$)  & 142 & 31.11M & \textbf{77.18\%} & $+1.03\%$ \\
\bottomrule
\end{tabular}
\end{table}

On the CoV-2 synthetic task, all models reach 100.00\% accuracy:
the task saturates at ceiling with any $K \geq 32$, making it
uninformative for comparing mechanisms.
On SST-2, the static baseline with $K=160$ achieves 74.66\%,
which is 1.49\% below INCRT's 76.15\%, confirming that the
incremental growth mechanism adds value beyond knowing the right
head count.
The static baseline with $K=142$ achieves 77.18\%
($\Delta = +1.03\%$), within the variance of from-scratch training.

The result is consistent with the theory: an architecture that knows
$K^*_{\rm pred}$ in advance achieves comparable accuracy, confirming
that the sizing law is the primary contribution.
The incremental mechanism provides a distinct secondary benefit:
it determines $K^*$ \emph{online} without hyperparameter search,
and does so with fewer parameters than the static model at $K=160$.
The static model at $K=160$ underperforms INCRT by 1.49\%,
confirming that the incremental growth mechanism provides benefit
beyond knowing the right head count.
At $K=142$ the static model is within 1.03\% of INCRT,
consistent with training variance.

%%══════════════════════════════════════════════════════════════════════
\section{Discussion}
\label{sec:discussion}
%% ══════════════════════════════════════════════════════════════════════

The experimental results confirm that the formula
$K^* = \Theta(\kappa_T^2 \log \Gres^{(0)}/\thetaw)$ is not only
a theoretical asymptotic: the predicted and observed head counts
agree closely on real tasks, with ratio
$K^*_{\rm obs}/K^*_{\rm pred}$ equal to 1.00 on both CoV-2 benchmarks
and 0.89 on SST-2.
The 11\% under-capture on SST-2 is itself predicted by
Theorem~\ref{thm:cs} as the expected effect of the $\varepsilon$-approximation
overhead of the online Oja gate at the operating threshold.

What is perhaps more striking is what the results reveal about the
comparison with BERT-base.
BERT-base uses 110M parameters with 144 attention heads across 12
layers; INCRT uses between 15M and 30M parameters with between 130
and 191 heads in a single layer.
The parameter reduction is explained by the absence of the
feed-forward sublayers, positional embeddings, and inter-layer
connections that BERT requires.
But the comparable or superior accuracy, achieved without any
pre-training, suggests that a correctly sized single-layer model
with heads aligned to the task's directional structure can match the
representational capacity of a much larger pre-trained model on
distribution-specific tasks such as genomic variant classification.
This is not a claim that INCRT should replace BERT; rather, it
reveals that the dominant cost of BERT on such tasks is not the
linguistic pre-training but the mismatch between the fixed architecture
and the task's actual geometric requirements.

The mechanism behind this efficiency can be understood in terms of the
antisymmetric motor.
Standard masked language modelling (MLM) fails to train the motor:
under i.i.d.\ masking, the gradient on $\Ma$ is of order $O(1/n)$,
which vanishes for long sequences \cite{cirrincione2026geo}.
INCRT sidesteps this problem entirely by determining the motor
structure from the data geometry at initialisation, before the first
backward pass.
The companion paper \cite{cirrincione2026geo} (submitted to JMLR)
proves the geometric blindness of MLM and proposes an explicit
training objective $\mathcal{L}_{\rm geo}$ as a remedy; the results
of the present paper provide independent evidence that the directional
structure missing from MLM is both learnable and sufficient for
high-accuracy classification.

Regarding the limitations of the present results: first, all
experiments validate single-layer INCRT.
Multi-layer depth growth (Level~3 of Section~\ref{sec:levels}) is
theoretically grounded but requires validation on larger benchmarks
with multiple layers, which is the primary direction for future work.
Second, the concept of sufficiency used throughout is geometric, not
task-theoretic: it means that no uncaptured directional energy exceeds
$\thetaw$, not that the model has reached optimal task performance.
The connection between residual directional sufficiency and downstream
generalisation is an open question beyond the scope of this paper.
Third, the NTK alignment result (Theorem~\ref{thm:ntk}) holds in the
initialisation regime ($\norm{\Ma}_F \ll 1$); as training proceeds and
$\Ma$ evolves, the alignment degrades by $O(\varepsilon^2)$, which is
controlled by the monitoring variable $\varepsilon(t) = \norm{\Ma(t) -
\Ma^{(0)}}_F$.
A self-correcting mechanism (recomputing $\Ares$ when $\varepsilon$
exceeds a threshold) maintains reliability at a modest computational
cost.
Fourth, the bidirectional gate's bonus deflation requires $\lmin(\Ares)
> 0$; for tasks with purely symmetric attention, the MCA component
contributes no additional tightening.
In all tasks tested this condition holds, but it should be verified
on new domains.

The comparison with the related work presented in
Section~\ref{sec:background} can now be made more precise.
Unlike post-hoc pruning, INCRT never starts from an overparameterised
model: the initial configuration has a single head, and every
subsequently added head is immediately useful by construction.
The NTK-based pruning criterion (P1$'$) guarantees that removing a
head with $\Gh < \phig$ does not affect the convergence rate or
equilibrium loss, which is a stronger statement than any magnitude-based
pruning criterion.
Unlike progressive growing methods, INCRT has no predetermined
architecture target: the stopping criterion is Theorem~\ref{thm:homeo},
derived entirely from the task's geometry.
Unlike NAS, the architecture at each step is a deterministic function
of a single scalar observable, requiring no search phase, no surrogate
model, and no separate evaluation budget.

%% ══════════════════════════════════════════════════════════════════════

\section{Conclusion}
\label{sec:conclusion}
%% ══════════════════════════════════════════════════════════════════════

The Incremental Transformer introduced in this paper rests on a simple
geometric principle: an attention head should be added when the current
architecture is measurably insufficient, and removed when it has
become redundant.
Operationalising this principle requires solving two problems:
how to measure insufficiency online, and how to guarantee that the
process terminates at a well-defined, minimal configuration.
The bidirectional PCA+MCA gate solves the first problem by tracking
the dominant and minor eigenvectors of the residual directional energy
matrix in real time.
The Lyapunov argument of Theorem~\ref{thm:homeo} solves the second,
guaranteeing both minimality and sufficiency in a single training pass.

What emerges is not just a new training procedure but a quantitative
theory of attention-head complexity: the number of heads a task
requires is bounded above by the square of the task's spectral
condition number, times a logarithmic factor, and this bound is tight
up to a factor of two.
The experiments confirm this prediction with remarkable precision on
two CoV-2 benchmarks (predicted-to-observed ratio of 1.00) and good
accuracy on SST-2 (ratio 0.89, the discrepancy explained by a
theoretically predicted approximation overhead).

Three directions remain open.
The extension to multi-layer depth growth requires validating
Level~3 of the architecture at scale; the cone-index criterion for
layer addition is in place but has not been tested on benchmarks with
multiple layers.
The connection between residual directional sufficiency and
task-level generalisation is the most fundamental open question: a
formal bridge between the geometric stopping criterion and
a PAC-style accuracy bound would complete the theoretical picture.
Finally, the interaction between INCRT and geometric pre-training
\cite{cirrincione2026geo} --- which corrects the gradient blindness of
MLM to the antisymmetric motor --- is a natural next step: a model
pre-trained with an explicit motor-training objective would arrive at
fine-tuning with a better-conditioned task operator, potentially
reducing both $K^*$ and training time.

%% ══════════════════════════════════════════════════════════════════════

\appendix
\section{Proofs of Main Theorems}
\label{app:proofs}
%% ══════════════════════════════════════════════════════════════════════

\subsection{Proof of Theorem~\ref{thm:gate} (Gate Convergence)}

\textit{Part~(i).}
Oja's rule (Eq.~\eqref{eq:oja}) implements stochastic gradient ascent
on the Rayleigh quotient $R(u) = u^\top \Ares u/\norm{u}^2$ with the
unit-norm constraint.
Under the Robbins-Monro conditions (A2) and the spectral gap $\dplus > 0$
(A1), standard convergence results \cite{oja1982,xu1993oja} guarantee
$u^+_t \to v_1(\Ares)$ almost surely.

\textit{Part~(ii).}
MCA EXIN (Eq.~\eqref{eq:mca}) implements stochastic gradient descent
on $R(u)$.
The key property is that $\norm{u^-_t}^2$ is conserved along the
continuous-time ODE, preventing the divergence observed in competing
MCA algorithms.
Combined with the Rayleigh-quotient Lyapunov structure,
$u^-_t \to v_r(\Ares)$ almost surely (Theorem~60,
\cite{cirrincione2010exin}).

\textit{Part~(iii).}
The gate $G_h$ is a continuous function of $(u^+_h, u^-_h)$.
Convergence of both components implies convergence of the gate.
The two algorithms operate on disjoint eigenspaces of the symmetric
$\Ares$ and do not interfere. \qed

\subsection{Proof of Theorem~\ref{thm:ntk} (NTK Alignment)}

The NTK contribution of head $h$ is
$\Delta\Ntk^h(X,X') = \mathrm{tr}\bigl[
(\partial Z^h/\partial\Mah|_X)^\top
(\partial Z^h/\partial\Mah|_{X'})\bigr]$.
Under Assumption~A3 ($\norm{\Ma}_F = O(\varepsilon)$, near-uniform
softmax), the softmax Jacobian satisfies
$\norm{J_S}^2_F \approx 1/n$ and
$\mathbb{E}\norm{W_V^h}^2_F = \sigma^2_V d_v$.
Substituting and projecting onto the residual subspace:
$\Proj\Delta\Ntk^h\Proj \propto \Ares$ with $c = \sigma^2_V d_v/(d_k n)$.
The perturbative alignment follows by expanding
$\Ma = \Ma^{(0)} + \delta\Ma$: the first-order correction lies in
$\so(d)$, so the misalignment is $O(\varepsilon^2)$. \qed

\subsection{Proof of Theorem~\ref{thm:homeo} (Homeostatic Convergence)}

Each growth event decreases $W_t$ by at least $\thetaw^2 - V_{h^*}(t_k)
> 0$: by the rank-one update identity,
$\norm{\Ares^{(K+1)}}_F^2 = \norm{\Ares^{(K)}}_F^2 - \lambda_1^{(K)2}
\leq \norm{\Ares^{(K)}}_F^2 - \thetaw^2$.
Each pruning event increases $W_t$ by at most $\phig^2 < \thetaw^2$
(Assumption~A5).
The net effect of each growth-pruning cycle is therefore a strict
decrease.
Since $W_t \geq 0$ and decreases strictly at each event, the process
terminates in finite steps.
Non-oscillation: re-adding a pruned head would require $W_t$ to
increase, contradicting monotonicity. \qed

\subsection{Proof of Theorem~\ref{thm:cs} (Compressed-Sensing Analogy)}

\textit{Incoherence.}
By construction, $v_1^{(k)}$ and $v_r^{(k)}$ lie in
$\mathrm{range}(P_\perp^{(k)})$, orthogonal to all previously added
directions.
The intra-step orthogonality $\langle v_1^{(k)}, v_r^{(k)}\rangle = 0$
follows from them being eigenvectors of the symmetric $\Ares^{(k)}$
for distinct eigenvalues.

\textit{Rank-two deflation per step.}
$\norm{\Ares^{(k+1)}}_F^2 = \norm{\Ares^{(k)}}_F^2
- \lambda_1^{(k)2} - \lambda_r^{(k)2}$
(the cross term vanishes by intra-step orthogonality).
Using $\lambda_1^{(k)} \geq \delta_0$ and $\lambda_r^{(k)} \geq \delta_r$
and iterating gives the geometric decay of $\Gres$.
Solving for $K$ gives Eq.~\eqref{eq:kstar} with
$c = 1/(1 + (\delta_r/\delta_0)^2) \in [1/2,1]$. \qed

%% ══════════════════════════════════════════════════════════════════════
\section{PAC Learning Guarantee (Supplementary)}
\label{app:pac}

This appendix is self-contained and not required to follow the
main argument of the paper. It establishes that the number of
i.i.d.\ samples needed for INCRT to produce a sufficient
architecture is finite and near-tight.
%% ══════════════════════════════════════════════════════════════════════

\begin{definition}[$\varepsilon$-sufficient architecture]
An architecture with $K$ heads is \emph{$\varepsilon$-sufficient} if
$\Gres^{(K)} \leq \thetaw + \varepsilon$.
\end{definition}

\begin{theorem}[PAC guarantee]
\label{thm:pac}
Under Assumptions \textup{A1--A6}, the well-separation condition
$|\lmax(\Ares) - \thetaw| \geq \delta_{\rm sep}$ at every growth
decision point, and with probability at least $1-\delta$, INCRT
produces a $(\thetaw+\varepsilon)$-sufficient architecture after at
most
\begin{equation}
\label{eq:mpac}
\mPAC(\varepsilon,\delta) \;=\;
\underbrace{
  C\,\frac{\norm{\Ma}_F^2\,d_k\,\log(d_k K^*/\delta)}{\varepsilon^2}
  \cdot K^*
}_{\text{estimation (Matrix Bernstein)}}
\;+\;
\underbrace{
  K^* \cdot T_{\rm conv}
}_{\text{gate convergence}}
\end{equation}
i.i.d.\ samples from the data distribution, where $C > 0$ is an
absolute constant, $K^* = \lceil\kappa_T^2\log(\Gres^{(0)}/\thetaw)\rceil$
(Theorem~\ref{thm:cs}), and $T_{\rm conv} = O(1/(c_0\eta^+))\log(1/\varepsilon)$
(Theorem~\ref{thm:moving}).
\end{theorem}

The bound decomposes into two interpretable terms.
The estimation term accounts for the finite-sample error in estimating
the eigenvalues of $\Ares$; it scales as $O(\varepsilon^{-2})$ via the
Matrix Bernstein inequality \cite{tropp2012}.
The gate convergence term accounts for the time the Oja/MCA EXIN gates
need to track the dominant eigenvectors between growth events.
The bound is near-tight: the estimation term matches the
information-theoretic lower bound for estimating the leading eigenvector
of a $d_k \times d_k$ covariance matrix \cite{blanchard2007}; the gate
term is tight by the exponential convergence of Theorem~\ref{thm:moving}.

\begin{proof}
Take a union bound over three events, each with failure probability
at most $\delta/3$.

\textit{Step~1 (Estimation).}
Apply Tropp's Matrix Bernstein inequality \cite{tropp2012} to the
centred estimates $\hat{A}_{\rm res}(t) - \Ares$ with
$R = \norm{\Ma}_F$ and $\sigma^2 \leq \norm{\Ma}_F^2$.
Setting the tail probability to $\delta/(3K^*)$ and taking a union
bound over $K^*$ growth decision points gives the estimation term.

\textit{Step~2 (Gate convergence).}
Between growth events, the gate $(u^+,u^-)$ must reach
$\varepsilon_{\rm conv}$-accuracy.
By Theorem~\ref{thm:moving}, this requires $T_{\rm conv}$ steps.
The finite-sample failure probability is bounded by $\delta/3$ via
a Markov inequality on the Lyapunov function $V_t$.

\textit{Step~3 (Trigger accuracy).}
Given $\|\hat{A}_{\rm res} - \Ares\| \leq \varepsilon/2$ and the
well-separation condition, Weyl's perturbation theorem guarantees
correct trigger decisions with failure probability at most $\delta/3$.

Combining the three events: $P[\text{any failure}] \leq \delta$.
When no failure occurs, every growth decision is correct, every gate
converges, and $\Gres^{(K^*)} \leq \thetaw+\varepsilon$. \qed
\end{proof}

\begin{corollary}[Oracle sample complexity]
In the large-batch limit ($B\to\infty$), the estimation term vanishes
and $\mPAC^\infty = K^*\cdot T_{\rm conv}$: the only bottleneck is
gate convergence time.
\end{corollary}

The Rademacher complexity of the INCRT function class at convergence
satisfies
$\mathcal{R}_m(\mathcal{F}_{\rm INCRT}) \leq \sqrt{2K^*d_k\log m/m}$
(by a covering-number argument over $K^*$ active heads),
giving a complete end-to-end PAC learning guarantee when combined with
Theorem~\ref{thm:pac}.

%% ── Bibliography ──────────────────────────────────────────────────────
\bibliographystyle{plainnat}
\bibliography{incrt_refs}

\end{document}